\documentclass[12pt]{article}

\usepackage{PRIMEarxiv}
\usepackage{amsmath}
\usepackage{amsthm}
\usepackage[utf8]{inputenc} 
\usepackage[T1]{fontenc}    
\usepackage{hyperref}       
\usepackage{url}            
\usepackage{booktabs}       
\usepackage{amsfonts}       
\usepackage{nicefrac}       
\usepackage{microtype}      
\usepackage{lipsum}
\usepackage{fancyhdr}       
\usepackage{graphicx}       
\usepackage{multirow}

\usepackage{graphicx} 
\usepackage{float} 
\usepackage{subfigure} 

\usepackage{color}

\graphicspath{{media/}}     

\newtheorem{theorem}{Theorem}

\newtheorem{proposition}[theorem]{Proposition}

\title{RoTHP: Rotary
Position Embedding-based Transformer Hawkes Process}

\author{
  Anningzhe Gao$^*$, Shan Dai \thanks{The two authors are equally contributed}\\
  Shenzhen Research Institute of Big Data \\
  \texttt{\{gaoanningzhe, shandai\}@sribd.cn } \\
}

\begin{document}
\maketitle

\begin{abstract}

{Temporal Point Processes (TPPs), especially Hawkes Process are commonly used for modeling asynchronous event sequences data such as financial transactions and user behaviors in social networks. Due to the strong fitting ability of neural networks, various neural Temporal Point Processes are proposed, among which the Neural Hawkes Processes based on self-attention such as Transformer Hawkes Process (THP) achieve distinct performance improvement. Although the THP has gained increasing studies, it still suffers from the {sequence prediction issue}, i.e., training on history sequences and inferencing about the future, which is a prevalent paradigm in realistic sequence analysis tasks. What's more, conventional THP and its variants simply adopt initial sinusoid embedding in transformers, which shows performance sensitivity to temporal change or noise in sequence data analysis by our empirical study. To deal with the problems, we propose a new Rotary Position Embedding-based THP (RoTHP) architecture in this paper. Notably, we show the translation invariance property and {sequence prediction flexibility} of our RoTHP induced by the {relative time embeddings} when coupled with Hawkes process theoretically. Furthermore, we demonstrate empirically that our RoTHP can be better generalized in sequence data scenarios with timestamp translations and in sequence prediction tasks. }
\end{abstract}


\section{Introduction}

Many natural and artificial systems produce a large volume of discrete events occurring in continuous time, for example, the occurrence of crime events, earthquakes, patient visits to hospitals, financial transactions, and user behavior in mobile applications \cite{daley2007introduction}. It is essential to understand and model these complex event dynamics so that accurate analysis, prediction, or intervention can be carried out subsequently depending on the context.

The occurrence of asynchronous event sequences is often modeled by temporal point processes \cite{cox1980point}. They are stochastic processes with (marked) events on the continuous time domain. One special but significant type of temporal point processes is the Hawkes process \cite{hawkes1971spectra}. A considerable amount of studies have used the Hawkes process as a standard tool to model event streams, including:
discovering of patterns in social interactions \cite{guo2015bayesian}, personalized recommendations based on users’ temporal behavior \cite{du2016recurrent}, construction and inference on network structure \cite{yang2013mixture,choi2015constructing} and so on. Hawkes processes usually model the occurrence probability of an event with a so-called intensity function. For those events whose occurrence are influenced by history, the intensity function is specified as history-dependent.
However, the simplified assumptions implied by Hawkes process for the complicated dynamics limit the models’ practicality. As an example, Hawkes process states that all past events should have positive influences on the occurrence of current events. Meanwhile, the lack of various nonlinear operations in traditional Hawkes process also sets an upper limit for Hawkes process’ expressive ability.

Thus there emerge some recent efforts in increasing the expressiveness of the intensity function using nonparametric models like kernel methods and splines \cite{zhou2013learning,wahba1990spline} and neural networks\cite{shchur2021neural}, especially recurrent neural networks \cite{du2016recurrent,mei2017neural} due to the sequence modeling ability. RNN-based Hawkes processes use a recurrent structure to summarise history events, either in the fashion of discrete time \cite{du2016recurrent,xiao2017modeling} or continuous-time \cite{mei2017neural}. This solution not only makes the historical contributions unnecessarily additive but also allows the modeling of complex memory effects such as delays. However,  these RNN-based Hawkes processes also inherit the drawbacks of RNN, for instance, it may take a long time for the patient to develop symptoms due to certain sequel, which has obvious long-term characteristics, such as diabetes, cancer and other chronic diseases, while these RNN-based models are hard to reveal the long-term dependency between the distant events in the sequences \cite{bengio1994learning}.
Other likelihood-free methods such as using reinforcement learning \cite{li2018learning} and generative adversarial networks \cite{xiao2017wasserstein} to help deal with the complex asynchronous event sequences data are also investigated recently.

{Recent developments in natural language processing (NLP) have led to an increasing interest in the self-attention mechanism. \cite{zhang2020self} present self-attention Hawkes process, furthermore, \cite{zuo2020transformer}
propose transformer Hawkes process based on the attention mechanism and encoder structure in transformer. This model utilizes pure transformer structure without using RNN and CNN, and achieves state-of-the-art performance. Although the self-attention mechanism applied in the Hawkes process performs empirically superior to RNNs in processing sequences data, most of the existing attention-based Hawkes processes architecture \cite{zhang2020self,zuo2020transformer,zhang2021universal,yu2022transformer} still suffer from timestamp noise sensitivity problem and sequence prediction issue. Specifically, common asynchronous events sequences data generated in reality will be naturally accompanied by temporal noise such as timestamp translation or systematic accuracy deviation due to the limited recording capabilities and storage capacities. For example, in large wireless networks, the event sequences are usually logged at a certain frequency by different devices whose
time might not be accurately synchronized and whose accuracy varies within a big range.  Conventional THP and its variants simply adopt initial sinusoid embedding in transformers and underexplored the position and time encoding problem in neural Hawkes processes, which is, however, crucial for the modeling of asynchronous events. We will next also demonstrate the timestamp noise sensitivity problem empirically. What's more,  to increase the generalizability of transformer in long sequences, various position encoding methods\cite{he2020deberta,press2021train,su2024roformer}are proposed. The versatility and great success of large language model (LLM) in handling complex tasks and its potential need to work with longer texts  further stimulated more recent research on position coding \cite{chen2023extending,peng2023yarn} regarding the sequence prediction and generation issue.  However, existing attention-based Hawkes process models can't be easily used or generalized in sequences prediction tasks. We will next discuss the problem from the position encoding architecture perspective and illustrate it by using their empirical performance in the related prediction tasks. To overcome the aforementioned problems, we construct a new Rotary Position Embedding-based THP (RoTHP) architecture. By adaptively adjusting the position encoding method to accommodate the crucial temporal information and utilizing the interval characteristics of Hawkes process, we further show the translation invariance property and {sequence prediction flexibility} of our RoTHP induced by the {relative time embeddings} when coupled with Hawkes process theoretically. Additional simulation studies are also provided to illustrate the superior performance of our RoTHP compared to the existing attention-based Hawkes process models.}

{In a nutshell, the contributions of the paper are: (i) We propose a RoTHP architecture to deal with the timestamp noise sensitivity problem and sequence prediction issue suffered by existing attention-based Hawkes process models; (ii) We show the translation invariance property and {sequence length flexibility} of our proposed RoTHP induced by the {relative time embeddings} when coupled with Hawkes process theoretically; (iii) We demonstrate empirically that our RoTHP can be better generalized in sequence data scenarios with translation or noise in timestamps and sequence prediction tasks. }

\textbf{Temporal point processes:} A temporal point process (TPP) is a stochastic process whose realization consists of a sequence of discrete events localized in continuous time, $\mathcal{H}=\left\{t_i \in \mathbb{R}^{+} \mid i \in \mathbb{N}^{+}, t_i<t_{i+1}\right\}$ \cite{daley2007introduction}.  TPP can be equivalently represented as a counting process $N(t)$, which records the number of events that have happened till time $t$. We indicate with $\mathcal{H}_t:=\left\{t^{\prime} \mid t^{\prime}<t, t_i \in \mathbb{R}^{+} \right\}$ the historical sequence of events that happened before $t$.
Given an infinitesimal time window $[t, t+\mathrm{d} t)$, the intensity function of a TPP is defined as the probability of the occurrence of an event $ t^{\prime}$ in $[t, t+\mathrm{d} t)$ conditioned on the history of events $\mathcal{H}_t$:
\begin{equation}\label{eq1}
\lambda(t) \mathrm{d} t:  =p\left(t^{\prime}: t^{\prime} \in[t, t+\mathrm{d} t) \mid \mathcal{H}_t\right) \nonumber
	 =\mathbf{E}\left(\mathrm{d} N(t) \mid \mathcal{H}_t\right),
\end{equation}
where $\mathbf{E}\left(\mathrm{d} N(t) \mid \mathcal{H}_t\right)$ denotes the expected number of events in $[t, t+\mathrm{d} t)$ based on the history $\mathcal{H}_t$. Without loss of generality, we assume that two events do not happen simultaneously, i.e., $\mathrm{d} N(t) \in\{0,1\}$.

Based on the intensity function, it is straightforward to derive the probability density function $f(t)$ and the cumulative distribution function $F(t)$ \cite{rasmussen2018lecture}:
\begin{align}\label{eq2}
	\nonumber & f(t)=\lambda(t) \exp \left(-\int_{t_{i-1}}^t \lambda(\tau) \mathrm{d} \tau\right), \\
	& F(t)=1-\exp \left(-\int_{t_{i-1}}^t \lambda(\tau) \mathrm{d} \tau\right) .
\end{align}

A marked
TPP allocates a type to each event.
We indicate with $\mathcal{S}=\left\{\left(t_i,k_i \right)\right\}_{i=1}^n$ an event sequence, where the tuple $\left(t_i,k_i \right)$ is the $i$-th event of the sequence $\mathcal{S}$, $t_i$ is its timestamp, and $k_i \in \mathcal{U}$ is its event type.

\textbf{Hawkes Process}: An Hawkes process  models the self-excitation of events of the same type and the mutual excitation of different event types, in an additive way. Hence, the definition of the intensity function is given as:
\begin{equation}\label{eq3}
\lambda(t)=\mu+\sum_{t^{\prime} \in \mathcal{H}_t} \phi\left(t-t^{\prime}\right),
\end{equation}
where $\mu \geq 0$, named base intensity, is an exogenous component of the intensity function independent of the history, while $\phi(t)>0$ is an endogenous component dependent on the history. Besides, $\phi(t)$ is a triggering kernel containing the peer influence of different event types. To highlight the peer influence represented by $\phi(t)$, we write $\phi_{u, v}(t)$, which captures the impact of a historical type-$v$ event on a subsequent type-$u$ event \cite{farajtabar2014shaping}. In this example, the occurrence of a past type-$v$ event increases the intensity function $\phi_{u, v}\left(t-t^{\prime}\right)$ for $0<t^{\prime}<t$.

To learn the parameters of Hawkes processes, it is common to use Maximum Likelihood Estimation (MLE). Other advanced methods such as adversarial learning \cite{xiao2017wasserstein} and reinforcement learning \cite{li2018learning} methods have also been proposed. The log-likelihood of an event sequence $\mathcal{S}$ over a time interval $[0, T]$ is given by:
\begin{align}\label{eq4}
	\mathcal{L} & =\log \left(\prod_{i=1}^n f\left(t_i\right) (1-F(T))\right)\nonumber\\
	& =\log \left\{\prod_{i=1}^n \left[\lambda\left(t_i\right) \exp \left(-\int_{t_{i-1}}^{t_i} \lambda(\tau) \mathrm{d} \tau\right) \right] \exp \left(-\int_{t_{n}}^T \lambda(\tau) \mathrm{d} \tau\right) \right\}\nonumber\\
	& =\sum_{i=1}^n \log \lambda_{}\left(t_i\right)-\int_0^T \lambda(\tau) d \tau
\end{align}

{\section{Position Embedding in self-attention hawkes process and Discussions}}

\subsection{Existing position embeddings in self-attention hawkes process}

\subsubsection{Sinusoid embedding in THP}

THPs and its variants \cite{zhang2021universal,zhang2022temporal} characterize the temporal information by utilizing a temporal encoding
pattern analogous to the ‘‘vanilla’’ absolute positional
encoding \cite{vaswani2017attention} for each node in the cascade. Denote the temporal encoding as below:
\begin{equation}\label{eq5}
\left[\boldsymbol{x}\left(t_i\right)\right]_j=\left\{\begin{array}{l}
\cos \left(t_i / 10000^{\frac{j-1}{D}}\right), \text { if } j \text { is odd }, \\
\sin \left(t_i / 10000^{\frac{j}{D}}\right), \text { if } j \text { is even. }
\end{array}\right.
\end{equation}

We utilize cosine and sine function to obtain the temporal encoding for the time-stamp. For each $t_i$, we can get its corresponding temporal encoding: $\boldsymbol{x}\left(t_i\right) \in \mathbb{R}^D . D$ is the model dimensions we determined. And for the event type encoding, it can be obtained through the embedding matrix $\mathbf{K} \in \mathbb{R}^{D \times C}$, for each type of event, we set its corresponding one-hot vector $\boldsymbol{c}_i \in \mathbb{R}^C$, thus, we can get the embedding $\mathbf{K} \boldsymbol{c}_i$, the dimension of $\mathbf{K} \boldsymbol{c}_i$ is also $D$.

For the event sequence $s_n=\left\{\left(c_i, t_i\right)\right\}_{i=1}^{I_n}$, the corresponding temporal encoding and event encoding are $\boldsymbol{X}^T$ and $\left(\mathbf{K} \boldsymbol{C}_n\right)^T$, where $\boldsymbol{X}=\left\{\boldsymbol{x}\left(t_1\right), \boldsymbol{x}\left(t_2\right), \ldots, \boldsymbol{x}\left(t_{I_n}\right)\right\} \in \mathbb{R}^{D \times I_n}$ and $\boldsymbol{C}_n=\left[\boldsymbol{c}_1, \boldsymbol{c}_2, \ldots, \boldsymbol{c}_{I_n}\right] \in \mathbb{R}^{C \times I_n}$. It is worth noting that the temporal encoding matrix and event encoding matrix have same dimensions $\mathbb{R}^{I_n \times D}$, and each row of them correspond

\subsubsection{Time-shifted positional encoding in Self-Attentive Hawkes Process}

Self-Attentive Hawkes Process adopts a time-shifted positional encoding to inject order information to the modeling sequence. Specifically, for an event $\left(v_i, t_i\right)$, the positional encoding is defined as a $K$ dimensional vector such that the $k$-th dimension of the position embedding is calculated as:
\begin{equation}\label{eq6}
    Pos_{\left(v_i, t_i\right)}^k=\sin \left(\omega_k \times i+w_k \times t_i\right),
\end{equation}

where $i$ is the absolute position of an event in a sequence, and $\omega_k$ is the angle frequency of the $k$-th dimension, which is pre-defined and will not be changed. While $w_k$ is a scaling parameter that converts the timestamp $t_i$ to a phase shift in the $k$-th dimension. Multiple sinusoidal functions with different $\omega_k$ and $w_k$ are used to generate the multiple position values, the concatenation of which is the new positional encoding. Even and odd dimensions of the position embedding are generated from $\sin$ and $\cos$ respectively.

\subsection{Discussions}

{{\subsubsection{Timestamp noise sensitivity }}
Here, we first show that the relative time difference property shared by the Temporal Hawkes Processes' likelihood, which will cause the position embedding in the existing self-attention Hawkes process to be sensitive to timestamp noise or translations and thus largely affect its modeling performance to sequence data.

\begin{proposition} 
Let $\mathcal{H}=\left\{t_i \in \mathbb{R}^{+} \mid i \in \mathbb{N}^{+}, t_i<t_{i+1}\right\}$ be a Hawkes process with
conditional intensity $\lambda^*(t)$ as defined in \eqref{eq4}. If we observe all the arrival times
over the time period $[t_{1}, t_{n}]$, denoted as ${t_{1},...,t_{n}}$, then the log-likelihood function $\mathcal{L}$
for $\mathcal{H}_t$ is:

\begin{equation}\label{eq7}
\mathcal{L}=\sum_{i=2}^n\log \left[\mu+\sum_{j=1}^{i-1} \phi\left(t_i-t_j\right)\right]-\mu (t_n  - t_1)-\sum_{j=1}^{n-1}\int_0^{t_n-t_{j}} \phi(s) ds,
\end{equation}
which is a function of timestamp differences $t_i-t_j, i,j \in \{1,2,...,n\}$ and $i \neq j$. Namely, condition on $t_i-t_j$, $\mathcal{L}$ is independent of $t_i, i=1,...,n$.
\end{proposition}
\begin{proof}
   According to \eqref{eq4}, we have the log-likelihood function $\mathcal{L}$ for $\mathcal{H}_t$ over $[t_{1}, t_{n}]$:

\begin{equation}\label{eq18}
\mathcal{L}=\sum_{i=1}^n \log \lambda_{}\left(t_i\right)-\int_{t_{1}}^{t_{n}} \lambda(t) d t.
\end{equation}

Substituting \eqref{eq3}, we get
\begin{equation}\label{eq19}
	\mathcal{L}=\sum_{i=1}^n\log \left[\mu+\sum_{t^{\prime} \in \mathcal{H}_t} \phi\left(t_{i}-t^{\prime}\right)\right]-\int_{t_{1}}^{t_{n}}\left[\mu+\sum_{t^{\prime} \in \mathcal{H}_t} \phi\left(t-t^{\prime}\right)\right]dt,
\end{equation}

Since we only consider the arrival times
over the time period $[t_{1}, t_{n}]$, and $\sum_{t^{\prime} \in \mathcal{H}_t} \phi\left(t_{i}-t^{\prime}\right)=\sum_{j=1}^{i-1} \phi\left(t_i-t_j\right)$ for $i=2,...n$, we have 

\begin{equation}\label{eq20}
\sum_{i=1}^n\log \left[\mu+\sum_{t^{\prime} \in \mathcal{H}_t} \phi\left(t_{i}-t^{\prime}\right)\right]=\sum_{i=2}^n\log \left[\mu+\sum_{j=1}^{i-1} \phi\left(t_i-t_j\right)\right].
\end{equation}

Also, we have
\begin{equation}\label{eq21}
\int_{t_{1}}^{t_{n}}\left[\mu+\sum_{t^{\prime} \in \mathcal{H}_t} \phi\left(t-t^{\prime}\right)\right]dt=\mu (t_n  - t_1)+\int_{t_{1}}^{t_{n}}\left[\sum_{t^{\prime} \in \mathcal{H}_t} \phi\left(t-t^{\prime}\right)\right]dt,
\end{equation}

Note that 
$[t_{1}, t_{n}]=\left[t_{1}, t_2\right] \cup\left(t_2, t_3\right] \cup \cdots \cup\left(t_{n-2}, t_{n-1}\right] \cup\left(t_{n-1}, t_{n}\right]
$, and therefore

\begin{align}\label{eq22}
	\int_{t_1}^{t_n} \sum_{t^{\prime} \in \mathcal{H}_t} \phi\left(t-t^{\prime}\right) \mathrm{d}t&=\int_{t_1}^{t_2} \sum_{t^{\prime} \in \mathcal{H}_t} \phi\left(t-t^{\prime}\right) \mathrm{d} t+\int_{t_2}^{t_3} \sum_{t^{\prime} \in \mathcal{H}_t} \phi\left(t-t^{\prime}\right) \mathrm{d} t\nonumber\\& \quad +\ldots+\int_{t_{n-1}}^{t_n} 
	\sum_{t^{\prime} \in \mathcal{H}_t} \phi\left(t-t^{\prime}\right) \mathrm{d} t.
\end{align}

For the first two terms in the right hand, we get
\begin{equation}
\nonumber
\int_{t_1}^{t_2} \sum_{t^{\prime} \in \mathcal{H}_t} \phi\left(t-t^{\prime}\right) d t=\int_{t_1}^{t_2} \phi\left(t-t_1\right) d t,
\end{equation}

\begin{equation}\label{eq23}
\int_{t_2}^{t_3} \phi\left(t-t^{\prime}\right) d t=\int_{t_2}^{t_3} \phi\left(t-t_1\right) d t+\int_{t_2}^{t_3} \phi\left(t-t_2\right) d t,
\end{equation}

So, we have
\begin{equation}\label{eq24}
\int_0^{t_3} \sum_{t^{\prime} \in \mathcal{H}_t} \phi\left(t-t^{\prime}\right) d t=\int_{t_1}^{t_3} \phi\left(t-t_1\right) d t+\int_{t_2}^{t_3} \phi\left(t-t_2\right) d t.
\end{equation}

Similarly, we get
\begin{align}\label{eq25}
	\int_{t_1}^{t_n} \sum_{t^{\prime} \in \mathcal{H}_t} \phi\left(t-t^{\prime}\right) dt \nonumber& =
	\int_{t_1}^{t_n} \phi\left(t-t_1\right) d t+\int_{t_2}^{t_n} \phi\left(t-t_2\right) d t+\ldots+\int_{t_{n-1}}^{t_n} \phi\left(t-t_{n-1}\right) d t \nonumber\\&=\int_0^{t_n-t_1} \phi(s) d s+\int_0^{t_n-t_2} \phi(s) d s+\ldots+\int_{t_1}^{t_n-t_{n-1}} \phi(s)ds
 \nonumber\\& =\sum_{j=1}^{n-1}\int_0^{t_n-t_{j}} \phi(s) ds,
\end{align}
Combining \eqref{eq19},\eqref{eq20},\eqref{eq21},\eqref{eq25}, we complete the proof.
\end{proof}

By the above proposition, it is evident that the log-likelihood estimation for Temporal Hawkes Processes (THPs) is solely reliant on relative positions (timestamps), denoted as $t_{i} - t_{i-1}$ for $i = 2,3,...,n$. This suggests that if we substitute the temporal sequence with $$\{t_1+\sigma,t_2+\sigma,...,t_n+\sigma\}$$ the likelihood remains invariant (refer to Section \ref{discussion} for a detailed discussion). However, the temporal encoding in THP for the temporal sequence after the translation is
\begin{equation}\label{eq_pos}
Pos(t_1, t_2, ...,t_n)=\left\{\begin{array}{l}
\cos \left(\frac{t_i+\sigma}{10000^{\frac{j-1}{D}}}\right), \text { if } j \text { is odd }, \\
\sin \left(\frac{t_i+\sigma}{10000^{\frac{j}{D}}}\right), \text { if } j \text { is even. }
\end{array}\right.
\end{equation}
It is important to note that $Pos(t_1, t_2, ...,t_n)\neq Pos(t_1+\sigma,t_2+\sigma,...,t_n+\sigma)$, which presents an inconsistency with the likelihood,  generally being the main part of the loss function for neural Hawkes process modeling and training. Consequently, for the current timestamp encoding issue, we turn our attention to the relative positional encoding method, a technique extensively employed in Natural Language Processing (NLP) tasks. Specifically, we focus on the Rotary Positional Encoding (RoPE). Our aim is to adapt the RoPE for use in the context of neural Hawkes Point Processes.

{{\subsubsection{ Sequence prediction issue}}
The ability to predict future events based on past data is a critical aspect of many fields, including finance, healthcare, social media analytics, and more. In the context of temporal point processes, this predictive capability becomes even more vital. }

The importance of future prediction in temporal point processes is underscored by the need for stability and sensitivity to temporal translations. In the Hawkes process, the Transformer Hawkes Process (THP) is sensitive to temporal translations, meaning that small changes in the input sequence can lead to significant changes in the output. This sensitivity can be problematic when trying to predict future events based on past data, as minor variations in the input can lead to inaccurate predictions. 

In practical applications, such as financial transactions, the ability to accurately predict future events can have significant implications. For instance, in finance, accurate predictions can inform investment strategies and risk management decisions. In social media analytics, they can help anticipate user behavior and trends. Therefore, we need a stable model architecture under translations to predict future features.

\vspace{4pt}

\section{Proposed Model}
We will next introduce the proposed Position Embedding-based transformer Hawkes process and its properties.
\subsection{Rotary
Position Embedding-based transformer Hawkes Process }

\subsubsection{Model architecture}
The key idea in our model design is to apply the rotary position embedding method \cite{su2024roformer} into temporal process. We fix our notations: Let $M$ be the embedding dimension, $K$ be the number of events. Let $\mathcal{S} = \{(t_i,k_i)\}_{i = 1}^n$ be a sequence of Hawkes process.

We use $\mathbf{X}$ to denote the matrix representing the one-hot vector corresponding to the event sequence. $\mathbf{X}\in\mathbb{R}^{K\times L}$, the $i$th column of $\mathbf{X}$ is a one-hot vector where the $j$th entry is non-zero if and only if $k_i = j$. We train en event embedding matrix $W^E$, its $i$th column is the embedding of the $i$th event. Hence the embedding of the event is given by $\mathbf{Y} = W^E\mathbf{X}$.

The transformer models have no position information. Unlike the absolute positional embedding used in THP, we consider the $\mathbf{Rotary\ Temporal\ Positional\ Embedding}$ (RoTPE). Let $W^Q, W^K, W^V$ be the linear transformations corresponds to the $\mathbf{Q}, \mathbf{K}, \mathbf{V}$ vectors, i.e. 
\begin{align}\label{eq8}
    \mathbf{Q} \nonumber
	& = \mathbf{Y}W^Q, \\ \mathbf{K} \nonumber
	& = \mathbf{Y}W^K, \\\mathbf{V} 
	& = \mathbf{Y}W^V,
\end{align}

where $W^Q, W^K \in\mathbb{R}^{M\times M_Q}$, $W^V\in\mathbb{R}^{M\times M_V}$, $M_K$ in the dimension of the query embedding which is an even number and $M_V$ is the dimension of the value vector. Set
\begin{equation}\label{eq9}
\theta_j = 10000^{-2(j-1) / d},  \quad j = 1, 2,..., d/2,
\end{equation}
Let

$$
\begin{pmatrix}
\cos t_i\theta_1 & \sin t_i\theta_1 & ...& 0 & 0\\
-\sin t_i\theta_1 & \cos t_i\theta_1& ...& 0 & 0\\
...&...&...&...&...\\
0&0&...&\cos t_i\theta_{d/2} & \sin t_i\theta_{d/2}\\
0&0&...&-\sin t_i\theta_{d/2} & \cos t_i\theta_{d/2}\\
\end{pmatrix}
$$
be the rotary matrix $R_{t_i}$. We can see that $R_{t_i}^TR_{t_j}$ is just
$R_{t_j-t_i}$, measuring the relative temporal difference at position $i,j$. The attention matrix $A$ is given by
\begin{align}\label{eq10}
q_i^TR_{t_i}^TR_{t_j}k_j \nonumber &= y_i^T(W^Q)^TR_{t_i}^TR_{t_j}W^Ky_j
\\\nonumber&=y_i^T(W^Q)^TR_{t_j-t_i}W^Ky_j
    \\&=q_i^TR_{t_j-t_i}k_j,
\end{align}
which is only related to the time difference. The attention output is given by
\begin{equation}\label{eq11}
O = Softmax(\frac{A}{\sqrt{D_K}})V
\end{equation}
using the normalization. Then we apply a feed-forward neural network to get the hidden representation $\mathbf{h}(t_j)$ for $1\leq j\leq n$. Hence we can see from the construction that our RoTHP model only depends on the relative position, which is coincide with the loss function: If we modify our marked temporal sequence $\{(t_1,k_1), (t_2,k_2),...,(t_n,k_n)\}$ to $\{(t_1 +\sigma,k_1), (t_2 +\sigma,k_2),...,(t_n +\sigma,k_n)\}$ with a translation, both the loss function and our model output keep the same. 

Since the rotary matric are orthogonal, $R_iq_i$ and $R_jk_j$ will have the same length with $q_i,k_j$, this will make our training more stable. 

\subsubsection{Training}

The intensity function of Neural Hawkes process is given by
\begin{equation}\label{eq12}
\lambda(t) = \sum_{k=1}^K\lambda_k(t),
\end{equation}
where $\lambda_k$ is the intensity function of the $k$-th event, and 

\begin{equation}\label{eq13}
\lambda_k=f_k(\alpha_k(t-t_j) + \mathbf{w}_k^T\mathbf{h}(t_j) + b_k),
\end{equation}
in which $t$ is defined on interval $t \in\left[t_j, t_{j+1}\right)$, and $f_k(x)=\beta_k \log \left(1+\exp \left(x / \beta_k\right)\right)$ is the softplus function with "softness" parameter $\beta_k$. Different from the conditional intensity function setting in the THP \cite{zuo2020transformer}, here we directly adopt the time difference $t-t_{j}$ without the normalization by $t_{j}$. We will illustrate the translation invariance property under this modified form in the next subsection and show its superior performance compared to the original THP by experiments in Section \ref{exp}.

For the prediction of next event type and timestamp, we train two linear layers $W^{e}, W^t$
\begin{align}\label{eq14}
    \hat{k}_{j+1} \nonumber&=argmax(Softmax(W^e\mathbf{h}(t_j)),\\
\hat{t}_{j+1} &= W^t\mathbf{h}(t_j).
\end{align}

For the sequence $\mathcal{S}$, we define
\begin{align}\label{eq15}
\mathcal{L}_{event}(\mathcal{S}) \nonumber&= \sum_{j=1}^{n-1}-\log(Softmax(W^e\mathbf{h}(t_j))_{k_{j+1}}),\\
\mathcal{L}_{time}(\mathcal{S}) &= \sum_{j = 1}^{n-1}((t_{j+1} - t_j) - (\hat{t}_{j+1} -\hat{t}_j))^2,
\end{align}
where ${t}_j$ is the true time stamp of the event $j$. By definition,  $\mathcal{L}_{event}$ measures the accuracy of the event type prediction and $\mathcal{L}_{time}$ measures the mean square loss the of time prediction. Denote the log-likelihood of $\mathcal{S}$ as $\mathcal{L}$,then the training loss can be defined by
\begin{equation}\label{eq16}
\mathcal{L}(\mathcal{S}) = -\mathcal{L} + \beta_1\mathcal{L}_{event}(\mathcal{S}) + \beta_2\mathcal{L}_{time}(\mathcal{S}),
\end{equation}
where $\beta_1,\beta_2$ are hyper-parameters to control the range of event and time losses.

\subsection{Translation Invariance Property\label{discussion}}

In this subsection, we will first show the interval characteristics of the Hawkes process and then illustrate the translation invariance property of our proposed RoTHP. 

\begin{proposition}(Translation Invariance Property): We indicate with $\mathcal{S}=\left\{\left(t_i,k_i \right)\right\}_{i=1}^n$ a Neural Hawkes process where the tuple $\left(t_i,k_i \right)$ is the $i$-th event of the sequence $\mathcal{S}$, $t_i$ is its timestamp, and $k_i \in \{1,2,...,K\}$ is its event type, with intensity function as given in \eqref{eq12} and \eqref{eq13}. Set $$\mathcal{S}_\sigma = \{(t_1 +\sigma,k_1), (t_2 +\sigma,k_2),...,(t_n +\sigma,k_n)\}$$
be the modified sequence $\mathcal{S}$ with a translation. 
Then we have the following translation invariance property for our proposed RoTHP:
\begin{equation}\label{eq17}
	\mathcal{L}(\mathcal{S}) = \mathcal{L}(\mathcal{S_\sigma})
\end{equation}
\end{proposition}

\begin{proof}
    Observe that the translation doesn't affect the value of timestamp differences $t_i-t_j$, so it's enough for us to prove that  $\mathcal{L}(\mathcal{S})$ is a function of  $t_i-t_j$, where $i,j \in \{1,2,..., n\}$ and $i \neq j$. Namely, condition on $t_i-t_j$, $\mathcal{L}(\mathcal{S}) $ is independent of $t_i, i=1,...,n$.

Since $\mathcal{L}(\mathcal{S})=-\mathcal{L} + \beta_1\mathcal{L}_{event}(\mathcal{S}) + \beta_2\mathcal{L}_{time}(\mathcal{S})$, next we complete the proof for $\mathcal{L}$, $\mathcal{L}_{event}(\mathcal{S})$, and $\mathcal{L}_{time}(\mathcal{S})$ respectively.

\noindent For $\mathcal{L}$, we have 
\begin{align}
	\mathcal{L} & =\log \left[\prod_{i=1}^{{n}}f\left(t_i\right) f\left(k=k_i \mid t_i\right)\right] \nonumber\\
	& =\log \left[\prod_{i=1}^n \lambda\left(t_i\right) \exp \left(-\int_{t_{i-1}}^{t_i} \lambda(\tau) \mathrm{d} \tau\right) \frac{\lambda_{k_i}\left(t_i\right)}{\lambda\left(t_i\right)}\right] \nonumber\\
	& =\sum_{i=1}^{n}  \log \lambda_{k_i}\left(t_i\right)-\int_{t_{1}}^{t_{n} }\lambda(\tau) d \tau.
\end{align}

By construction of RoTHP, the attention is calculated by the inner product of the $Q$ vector and $K$ vector, we have:
\begin{align}\label{eq10}
	q_i^TR_{t_i}^TR_{t_j}k_j \nonumber &= y_i^T(W^Q)^TR_{t_i}^TR_{t_j}W^Ky_j
	\\\nonumber&=y_i^T(W^Q)^TR_{t_j-t_i}W^Ky_j
	\\&=q_i^TR_{t_j-t_i}k_j,
\end{align}
which is function of $t_i-t_j$. Since the $\bold{h(t_j)}$ involved in $\mathcal{L}$ is generated by feed-forward neural network with the attention output fed through,  
So the $\bold{h(t_j)}$  is also function of $t_i-t_j$, and we only need to consider the term $t-t_j$ involved in $\mathcal{L}$. Since the intensity function is chosen as in \eqref{eq12} and \eqref{eq13}, which we have proved for general $\phi\left(t-t_{j}\right)$  as in the proof of Proposition 1, $\mathcal{L}$ is also a function of $t_i-t_j$.

For $\mathcal{L}_{event}(\mathcal{S})$, the $\bold{h(t_j)}$  has been proved to be function of $t_i-t_j$;

Similarly for  $\mathcal{L}_{time}(\mathcal{S})$, it involves the terms $t_i-t_j$ and $\hat{t}_{j+1} -\hat{t}_j$. Since $\hat{t}_{j+1} = W^t\mathbf{h}(t_j)$ , $\hat{t}_{j+1} -\hat{t}_j$ is also function of $t_i-t_j$.

Thus we complete the proof.
\end{proof}

This implies that the timestamp translation will not change the loss function under the modeling of our RoTHP.

For the positional encoding in the THP and its variants, the temporal embedding and event embedding is added before the attention mechanism, hence the model output will change a lot after the translation. This implies that RoTHP is more appropriate for the modeling of gereral Hawkes process.

\textbf{Robustness to random noise}: we consider the noise given by translation and Gaussian noise. When the translation is applied to the temporal sequence, the model should give the same output since we only consider the temporal interval $[t_1, t_n]$. RoTHP satisfies this property. However, for the absolute positional encoding used in THP, change of the temporal sequence by translation will change the positonal encoding and hence change the model output. We will show that as the translation $\sigma$ changes, the behavior of THP is unstable.

We also consider the Gaussian noise acting on the temporal sequence. For a temporal sequence $\{(t_1, k_1), (t_2, k_2),...,(t_n,k_n)\}$, we consider the new sequence $\Tilde{\mathcal{S}} = \{(t_1 + \sigma_1, k_1), (t_2 + \sigma_2, k_2),...,(t_n + \sigma_n,k_n)\}$ where $\sigma_i\sim \mathcal{N}(0,\epsilon)$. We will show later that our proposed RoTHP will also have better performance under the Gaussian noise.

\subsection{Sequence Prediction Flexibility}

In multiple Natural Language Processing (NLP) tasks, RoPE has shown the extension property, which means that it can deal with longer sequences. Here we consider the case when we want to use the previous information to predict the future ones. For a temporal sequence $\mathcal{S}=\{(t_1,k_1),(t_2,k_2),...,(t_n,k_n)\}$, we pick an index $1<m<n$, and set $\mathcal{S}_{[1:m]} = \{(t_1,k_1),(t_2,k_2),...,(t_m,k_m)\}$ as the training sample, and $\mathcal{S}_{[m+1:n]} = (t_{m+1},k_{m+1}),(t_{m+2},k_{m+2}),...,(t_n,k_n)$ as the testing sample. This is the case when we need to predict the future case, which is really useful in the real world.

Intuitively, the RoTHP is much more stable in the time translations compared to THP and its variants, which may also lead to its better performance in future event prediction tasks for the Hawkes process. We next first give an analysis here. Consider the input of the models for THP and RoTHP. The training input for THP is the sequence $\mathcal{S}_{[1:m]}=\{(t_1,k_1),(t_2,k_2),...,(t_m,k_m)\}$ and the testing input is the sequence $\mathcal{S}_{[m+1:n]} = (t_{m+1},k_{m+1}),(t_{m+2},k_{m+2}),...,(t_n,k_n)$, so these two sequences have totally different temporal distribution since the model hasn't seen the future time stamps during the training process, which is the essential obstacle of THP in this case. However, for RoTHP, by the translation invariant property, we may view the training input and testing input as $\mathcal{S}_{[1:m]}=\{(0,k_1),(t_2-t_1,k_2),...,(t_m-t_1,k_m)\}$ and $\mathcal{S}_{[m+1:n]} = (0,k_{m+1}),(t_{m+2}-t_{m+1},k_{m+2}),...,(t_n-t_{m+1},k_n)$, the model thus can extend the distribution of the relative timestamps in the training set to the testing set. So RoTHP can lead to a better prediction. We will show the results for the experiments of future prediction using RoTHP and THP in Section \ref{Predictability}.

\section{Experiment}

\subsection{Dataset}

Our experiments were conducted on three distinct datasets. These datasets were specifically chosen due to their long temporal sequences, which are ideal for our study on encoding methods. Furthermore, the Synthetic dataset, generated from a Hawkes process, provides a suitable case for our investigation.

\paragraph{Synthetic} This dataset was generated using Python, following the methodology outlined in \cite{zhang2020self}. It is a product of a Hawkes process, making it a suitable case for our study. Our synthetic dataset admits 5 event types, with average length 60. The minimal length is 20 and the maximal length is 100.

\paragraph{Financial Transactions}\cite{du2016recurrent} This dataset comprises stock transaction records from a single trading day. The sequences in this dataset are lengthy, and the events are categorized into two types: "Buy" and "Sell". The average length of the dataset is 2074 and is appropriate to our experiment.

\paragraph{StackOverFlow}\cite{leskovec2016snap} This dataset is a collection of data from the question-answer website, Stacoverflow. We consider the history of user interactions as a temporal sequence. The average length of the sequences in the dataset is 72, with minimum 41 and maximum 736, and it has 22 event types.

\paragraph{Retweet}\cite{zhao2015seismic}  The data set for Retweets compiles a variety of tweet chains. Every chain comprises an original tweet initiated by a user, accompanied by subsequent response tweets. The accompanying information includes the timing of each tweet and the user's identifier. The average length of the sequences is 109, with minimum 50 and maximum 264. The event types are separated into 3 different types depending on number of the followers: "small", "medium" and "large"

\paragraph{Memetrack}\cite{leskovec2016snap} This dataset comprises references to 42,000 distinct memes over a period of ten months. It encompasses data from more than 1.5 million documents, including blogs and web articles, sourced from over 5,000 websites. Each sequence in the dataset represents the lifespan of a specific meme, with each occurrence or usage of the meme linked to a timestamp and a website ID.

\paragraph{Mimic-II}\cite{johnson2016mimic} The MIMIC-II medical dataset compiles data from patients' admissions to an ICU over a span of seven years. Each patient's visits are considered distinct sequences, with each sequence event marked by a timestamp and a diagnosis.

\subsection{Baselines}
In our research, we juxtapose our devised model with four separate models. The primary focus, however, is concentrated on comparing our model and the THP, in an attempt to comprehend the influence of the rotary temporal positional encoding.

The first model, Recurrent Marked Temporal Point Process (RMTPP)(\cite{du2016recurrent}), employs a Recurrent Neural Network (RNN) architecture in predicting the temporal occurrence of the subsequent event. The Neural Hawkes process (NHP) (\cite{mei2017neural}), infuses neural networks with the Hawkes process to phenomenalize the prediction accuracy. In \cite{zhang2020self} propound the Self-attentive Hawkes Process (SAHP), which harnesses an attention mechanism in the Hawkes process predictions and incorporates further positional encoding in the model fabrication. Lastly, the Transformer Hawkes process (THP) as operatized in \cite{zuo2020transformer}, applies the transformer structure to the Hawkes process and incorporates absolute positional encoding in the model construction

\subsection{Implementation}

The architecture of our proposed RoTHP adopts similar transformer construction and chosen hyper-parameters for the Stackoverflow and financial transaction datasets . For the synthetic dataset, we use the construction of Set 1 in \cite{zuo2020transformer} and batch size 4, learning rate 1e-4. We will compare the result of THP using the same architecture and hyper-parameters for Synthetic dataset. 

\subsection{Result}

We can see that the RoTHP outperforms all other models for these three datasets for the log-likelihood, see Table \ref{table:log}\\
We evaluated the performance of various models, including RMTPP, NHP, SAHP, THP, and our proposed model. The evaluation metrics employed were log-likelihood and accuracy.

\paragraph{Log-Likelihood Analysis} Our model exhibited a significant advantage on the Financial dataset, achieving the highest log-likelihood value of 1.076, which is a clear indication of its superior fit compared to the other models. This suggests that our model is particularly adept at capturing the complex patterns and relationships inherent in financial data, which is crucial for accurate prediction and analysis in this domain.

On the SO dataset, our model's log-likelihood was 0.389. This high log-likelihood reflects a more accurate representation of the data's distribution, indicating that our model can effectively handle the intricacies of social interactions and networks.

In the case of the Synthetic dataset, our model once again demonstrated its robustness by attaining the highest log-likelihood value of 1.01. This strong performance on synthetic data, which is designed to mimic real-world scenarios, underscores the model's generalizability and its ability to adapt to various data structures.

For the Retweet dataset, we achieved 2.01, the only model larger than 0, also for the Memetrack dataset, we get the highest score of 1.71. These facts reflects our model's strong ability in the scenario of infomation spread in the social network.

For the Mimic-II dataset, our model outperforms all other models with log-likelihood 0.64. This implies our model's ability in the short temporal sequences case is still strong.




\begin{table}[h]
\centering
\caption{Log-likelihood}
\label{table:log}
\footnotesize
\begin{tabular}{@{}l|cccccc@{}}
\toprule
\textbf{Models} & \textbf{Financial} & \textbf{SO} & \textbf{Synthetic}& \textbf{Retweet}& \textbf{MemeTrick}&\textbf{Mimic-II} \\
\midrule
RMTPP & -3.89 & -2.6 & -1.33 & -5.99 & -6.04&-1.35\\ 
NHP & -3.6 & -2.55 & - & -5.6 & -6.23 & -1.38 \\ 
SAHP & - & -1.86 & 0.59 & -4.56 & - & -0.52\\ 
THP & -1.11 & -0.039 & 0.791 & -2.04 & 0.68 & 0.48\\ 
RoTHP & \textbf{1.076} & \textbf{0.389} & \textbf{1.01} & \textbf{2.01} & \textbf{1.71}&\textbf{0.64}\\
\bottomrule
\end{tabular}

\end{table}

\paragraph{Accuracy and RMSE}We consider the accuracy and Root Mean Square Error (RMSE) estimation on the following datasets: Financial, Mimic-II, and SO.

For the accuracy, we can see that the RoTHP model consistently outperforms the other models across all three datasets. It achieves the highest accuracy on the Financial dataset (62.26), the Mimic-II dataset (85.5), and performs comparably on the SO dataset (46.33) with the highest being THP at 46.4.

In terms of RMSE, which is a measure of error where lower values are better, RoTHP again outperforms all other models across all datasets. It achieves the lowest RMSE on the Financial dataset (0.60), the Mimic-II dataset (0.57), and the SO dataset (1.33).

The RoTHP model demonstrates superior performance in both accuracy and RMSE across all datasets when compared to the other models. This suggests that RoTHP is a more reliable and accurate model for these particular tasks. Its consistent performance across different types of data (Financial, Mimic-II, and SO) indicates its robustness and versatility.

\begin{minipage}[h]{\textwidth}
 \begin{minipage}[h]{0.45\textwidth}
  \centering
\makeatletter\def\@captype{table}\makeatother
                    \caption{Accuracy}
    \label{table:log} 
       \begin{tabular}{c|ccc} 

    \toprule
\textbf{Models} & \textbf{Financial} & \textbf{Mimic-II} & \textbf{SO} \\
\midrule
RMTPP & 61.95 & 81.2 & 45.9\\ 
NHP & 62.20 & 83.2 & 46.3 \\ 
THP & 62.23 & 84.9 & \textbf{46.4}\\ 
RoTHP & \textbf{62.26} & \textbf{85.5} & 46.33\\
\bottomrule
    \end{tabular}

  \end{minipage}
  \begin{minipage}[h]{0.45\textwidth}
   \centering

    \makeatletter\def\@captype{table}\makeatother
        \caption{RMSE}
      \label{table:acc}
         \begin{tabular}{c|ccc}        
          \toprule
\textbf{Models} & \textbf{Financial} & \textbf{Mimic-II} & \textbf{SO} \\
\midrule
RMTPP & 1.56 & 6.12 & 9.78\\ 
NHP & 1.56 & 6.13 & 9.83 \\ 
SAHP & - & 3.89 & 5.57 \\ 
THP & 0.93 & 0.82 & 4.99\\ 
RoTHP & \textbf{0.60} & \textbf{0.57} & \textbf{1.33}\\
\bottomrule
      \end{tabular}

   \end{minipage}

\end{minipage}

\paragraph{Comparison with THP}

In this study, we assess the performance of both THP and RoTHP based on their log-likelihood. For the training durations, RoTHP consistently outperforms THP, underscoring the benefits of rotary embedding in the Hawkes process. When applied to a financial transaction dataset, RoTHP exhibits a substantial enhancement at the sixth epoch. Furthermore, RoTHP demonstrates a significantly faster convergence rate for financial transactions, further highlighting its superior performance.

\begin{figure}[H] 
\centering 
\includegraphics[width=0.3\textwidth]{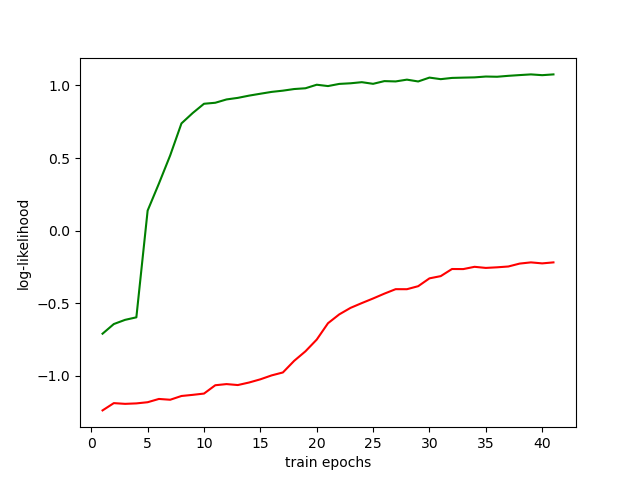} 
\includegraphics[width=0.3\textwidth]{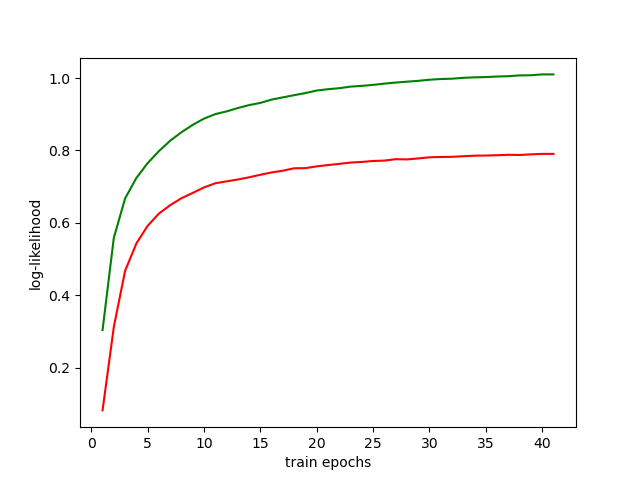} 
\includegraphics[width=0.3\textwidth]{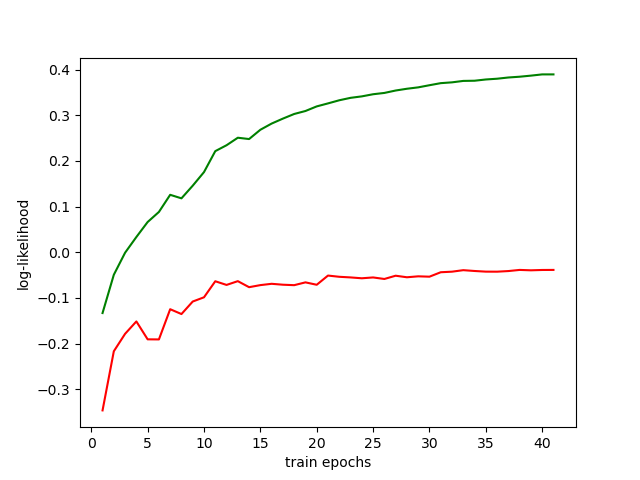} 
\caption{Comparison of log-likelihood between THP and RoTHP. Green lines are RoTHP, red lines are THP. The left figure represent the training process of financial dataset, the middle figure is for the synthetic dataset and the right is for the stackoverflow dataset. We can see in both figures RoTHP outperforms THP} 
\label{Fig.log} 
\end{figure}

\subsection{Robustness study\label{exp}}

In this section, we examine the robustness of both THP and RoTHP. As previously discussed, the loss function in the transformer Hawkes process is solely dependent on relative timestamps. We explore two scenarios: Translation and Gaussian noise.

In the Translation scenario, we apply a $\sigma$ translation to all timestamps in both the training and testing datasets. The model's behavior in financial transactions for $\sigma$ values ranging from 0 to 10 is presented in Table \ref{table:translation_financial}. Remarkably, the log-likelihood of RoTHP remains virtually unchanged, whereas that of THP fluctuates significantly. This observation suggests that RoTHP exhibits greater stability in this context and shows the translation invariant property of our method.

\begin{table}[h]
\centering
\caption{Translation log-likelihood of Financial, the table reports the difference between the log-likelihood after the translation and the origin log-likelihood}
\label{table:translation_financial}

\footnotesize
\begin{tabular}{@{}l|cccccccccc@{}}
\toprule
\textbf{Models} & \textbf{0} &  \textbf{0.2}&  \textbf{0.4}& \textbf{0.6}& \textbf{0.8}& \textbf{1}& \textbf{2}& \textbf{5}& \textbf{10} \\
\midrule
THP & 0  & 0.01  & 0.018  & 0.027  &0.035  & 0.044 & 0.101 & 0.014 & -0.078\\ 
RoTHP  & 0 & 0 & 0 &0 & 0 & 0 & 0 & 0 & 0\\
\bottomrule
\end{tabular}

\end{table}

In the Gaussian noise scenario, we introduce Gaussian noise $\sigma_i\sim \mathcal{N}(0,\epsilon)$ to each timestamp in every temporal sequence of the training dataset, thereby incorporating Gaussian noise into the training data. We set $\epsilon$ = 0.01 for our experiment. The results are displayed in Table \ref{table:noise}. We show the difference between the log-likelihood error before and after we add the Gaussian noise on the time stamps. Here, we choose $\epsilon$ small because the time gaps between adjacent events are small, and we must choose a noise smaller than the gaps. Otherwise, this will change the order of the events in the sequences. For convenience, we show the ratio of the absolute value of the difference and $\epsilon$ to magnify the influence of the Gaussian noise. We pick the Synthetic and SO datasets to see the influence of the Gaussian noise. The reason we pick these two datasets is because we need a long sequence length so that the temporal information will be more important. The Financial dataset has too small time stamp gaps, and the Retweet dataset has integer time stamps, which are not appropriate for this case.

\begin{table}[h]
\caption{Gaussian noise}
\label{table:noise}
\centering
\footnotesize
\begin{tabular}{l|ccc|cccccc}
\toprule
\multirow{2}{*}{\textbf{Models}}  & \multicolumn{3}{c|}{\textbf{Synthetic}}& \multicolumn{3}{c}{\textbf{SO}}\\
& \textbf{log-likelihood}& \textbf{Accuracy}& \textbf{RMSE}& \textbf{log-likelihood}& \textbf{Accuracy}&\textbf{RMSE} \\
\midrule
THP & 0.906 & 0 & 0.104 & 0.855 & 0.04 &4.379\\ 
Ours & \textbf{0.635} & 0 & \textbf{0.051} & \textbf{0.234} & \textbf{0.012} &\textbf{1.84}\\
\bottomrule
\end{tabular}

\end{table}

From Table \ref{table:noise}, we can see that RoTHP is much stabler than THP in the Gaussian noise case, implying the robustness of the method.

\subsection{Predict the future features\label{Predictability}}

In this subsection, we consider the case where we use the previous information to predict the future ones. Intuitively, the THP is sensitive to the translation, and RoTHP is more stable. Hence the RoTHP may perform better than THP. We do the test on financial transaction, synthetic and StackOverflow dataset, and Table \ref{table:future} shows the result.

\begin{table}[h]
\centering
\caption{Future prediction}
\label{table:future}
\footnotesize
\begin{tabular}{l|ccc|ccc|ccc@{}}
\toprule
\multirow{2}{*}{\textbf{Models}} & \multicolumn{3}{c|}{\textbf{Financial}} & \multicolumn{3}{c|}{\textbf{Synthetic}}& \multicolumn{3}{c}{\textbf{SO}}\\
&\textbf{log-likelihood} & \textbf{Accuracy}& \textbf{RMSE}& \textbf{log-likelihood}& \textbf{Accuracy}& \textbf{RMSE}& \textbf{log-likelihood}& \textbf{Accuracy}&\textbf{RMSE} \\
\midrule
THP & -1.04 & 0.6 & 0.661 & 0.088 & 0.38249 & {0.3405} & -0.942 & 0.4 &{2.954}\\ 
Ours & \textbf{1.24} & \textbf{0.62} & \textbf{0.654} & \textbf{1.08} & \textbf{0.3825} & \textbf{0.3298} & \textbf{0.476} & \textbf{0.403} &\textbf{2.872}\\
\bottomrule
\end{tabular}

\end{table}




The RoTHP model generally performs better than the THP model across all three datasets and metrics. The text also suggests that the THP model may be prone to overfitting, which could explain its poorer performance. Here the translation invariant property of RoTHP plays an important role in this case. By Figure \ref{Fig.future}, we can see that THP becomes even worse while RoTHP becomes better during the training process. Hence in this case THP meets the overfitting problem.

\begin{figure}[H] 
\centering 
\includegraphics[width=0.3\textwidth]{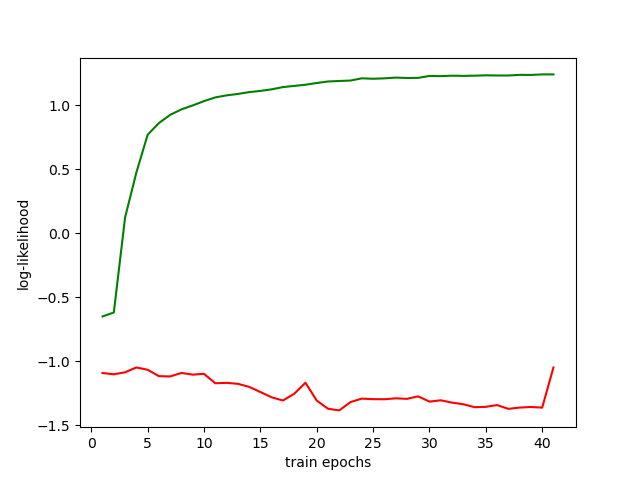} 
\includegraphics[width=0.3\textwidth]{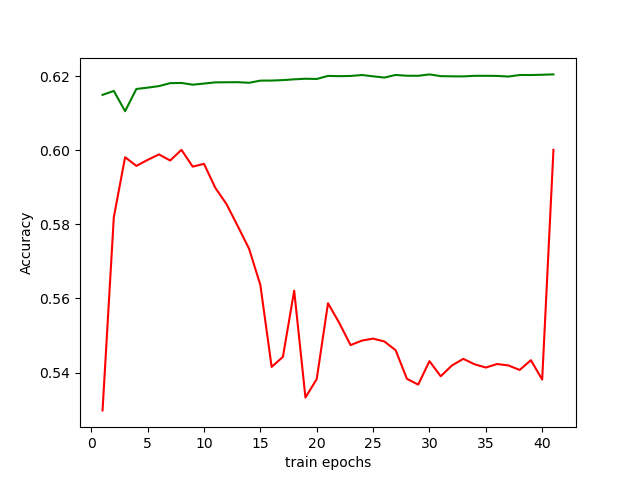} 
\includegraphics[width=0.3\textwidth]{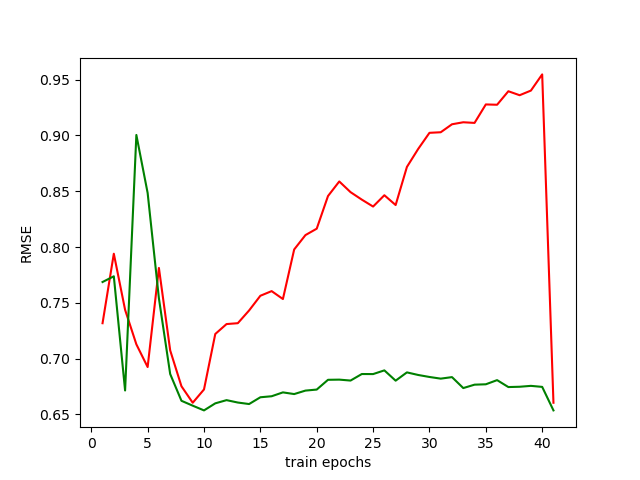} 
\caption{Green lines are RoTHP, red lines are THP. The experiment is for the financial transaction dataset} 
\label{Fig.future} 
\end{figure}

\section{Conclusion}
In this paper, we introduce a novel model architecture for the Hawkes process, known as RoTHP, which incorporates the rotary encoding method into the transformer Hawkes process. This results in a model that is both robust and high-performing. We provide an in-depth analysis of why rotary position encoding is more effective for the Hawkes process, examining the loss function and multi-head self-attention mechanism in detail. When tested on three long temporal sequence datasets - Synthetic, financial transaction, and StackOverflow - our model outperformed other models such as RMTPP, THP, NHP, and SAHP. Furthermore, RoTHP demonstrated impressive performance even under conditions of time stamp translation and Gaussian noise.

\bibliographystyle{unsrt}  
\bibliography{references}  

\end{document}